\documentclass{aims}

\usepackage{txfonts,cite}
\def\typeofarticle{Research article}
\def\currentvolume{31}
\def\currentissue{2}
\def\currentyear{2023}

\def\ppages{691--707.}
\def\DOI{10.3934/era.2023034}
\def\Received{22 September 2022}
\def\Revised{10 November 2022}
\def\Accepted{14 November 2022}
\def\Published{22 November 2022}

\numberwithin{equation}{section}

\usepackage[numbers,sort]{natbib}
\usepackage{enumerate}
\usepackage{graphicx}
\usepackage{amsmath}
\usepackage{amsthm}
\usepackage{booktabs}
\usepackage{multirow}
\usepackage{comment}
\usepackage{bbold}
\usepackage{xcolor}
\usepackage{dblfloatfix}
\usepackage{hyperref}
\usepackage{longtable}

\newtheorem{theorem}{Theorem}
\newtheorem{lemma}[theorem]{Proposition}
\usepackage[skip=0.333\baselineskip]{subcaption}
\begin{document}

\title{Time Delay Estimation of Traffic Congestion Propagation Due to Accidents Based on Statistical Causality}
\author{%
  YongKyung Oh\affil{1},
  JiIn Kwak\affil{2}
  and
  Sungil Kim\affil{1,2}\corrauth
}

\shortauthors{Authors}

\address{%
  \addr{\affilnum{1}}{Department of Industrial Engineering, Ulsan National Institute of Science and Technology, Ulsan, Republic of Korea}
  \addr{\affilnum{2}}{Artificial Intelligence Graduate School, Ulsan National Institute of Science and Technology, Ulsan, Republic of Korea}}

\corraddr{sungil.kim@unist.ac.kr; Tel: +82-52-217-3195.
}

\begin{abstract}
	The accurate estimation of time delays is crucial in traffic congestion analysis, as this information can be used to address fundamental questions regarding the origin and propagation of traffic congestion. However, the exact measurement of time delays during congestion remains a challenge owing to the complex propagation process between roads and high uncertainty regarding future behavior. To overcome this challenge, we propose a novel time delay estimation method for the propagation of traffic congestion due to accidents using lag-specific transfer entropy (TE). The proposed method adopts Markov bootstrap techniques to quantify uncertainty in the time delay estimator. To the best of our knowledge, our proposed method is the first to estimate time delays based on causal relationships between adjacent roads. We validated the method’s efficacy using simulated data, as well as real user trajectory data obtained from a major GPS navigation system in South Korea.
\end{abstract}

\keywords{
statistical causality; transfer entropy; time delay estimation; traffic trajectory data; traffic incident analysis
}

\maketitle

\section{Introduction}
Traffic congestion represents a universal problem for urban life owing to the dramatic growth in vehicle use, expansion of the economy and infrastructure, and proliferation of delivery services, among other factors. Traffic congestion frequently spreads into adjacent roads \citep{nguyen2016discovering}, resulting in greater damage to the overall traffic network.

	\begin{figure}[htb]
		\centering
		\includegraphics[width=0.5\linewidth]{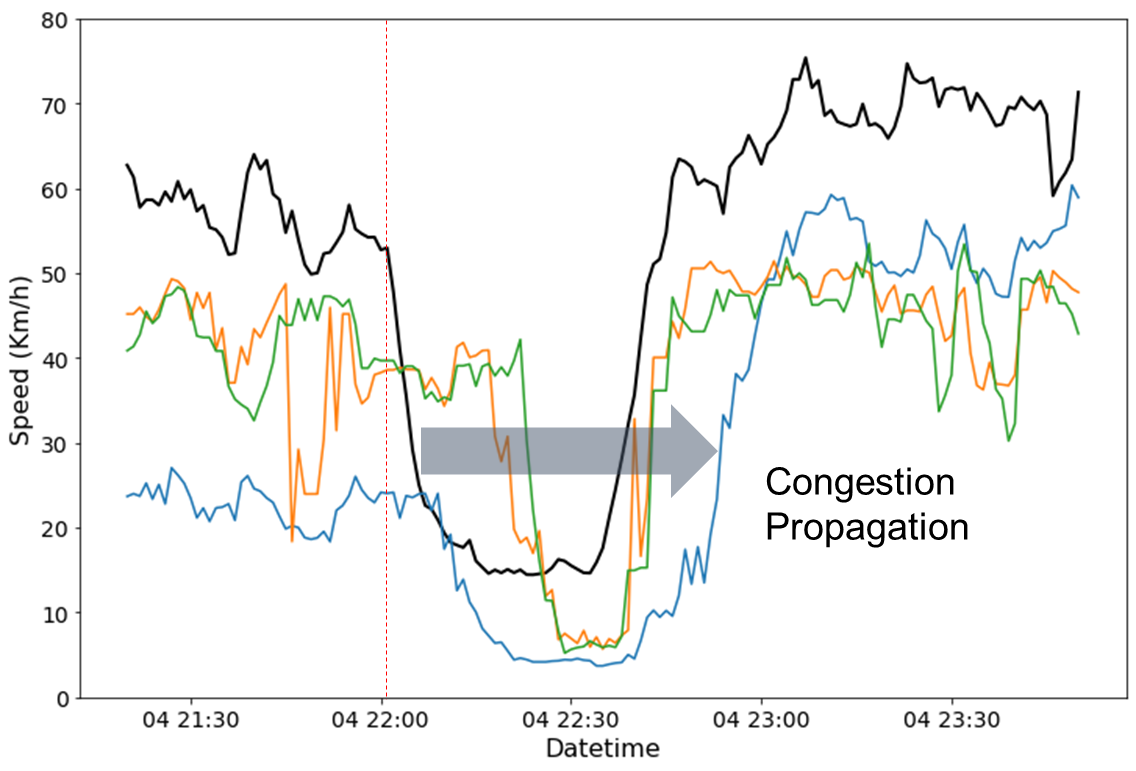}
		\caption{Motivating example: traffic congestion propagation}
		\label{fig:eta}
	\end{figure}
	
Consequently, the accurate estimation of time delays has become crucial in addressing fundamental questions regarding the origin points and propagation of traffic congestion. Figure \ref{fig:eta} from \cite{lee2022quantifying} illustrates this study’s objective by showing how the impact of a traffic accident propagates to incoming roads. The black solid line represents the average speed on the road segment where the accident occurs, while the orange, green, and blue solid lines represent the average speeds on the three adjacent incoming roads. The average speed on the road segment was computed from GPS trajectory data provided by the NAVER Corporation using a map-matching process \citep{newson2009hidden}. We can observe a time lag when the impact of an accident propagates to incoming roads. The time delay increases in the order of blue, green, and orange.

However, certain aspects of traffic congestion propagation pose statistical challenges to the accuracy of time delay estimation. First, the average speed on an incoming road is affected by various geographic and topological characteristics, such as road length and width, as well as road network topology. The impact of a traffic accident is distributed among all incoming roads in a complex pattern according to these characteristics. Furthermore, the duration of congestion is dynamic. As shown in the figure, the road denoted in blue exhibits a longer congestion duration than the other incoming roads. That is, the time delay in traffic congestion propagation does not simply denote a temporal pattern shift, but involves complicated temporal dynamics. Finally, data uncertainty is inherent in average road speeds. Trajectory-based average road speeds may be highly volatile depending on the availability of user data over a specific period.

To overcome the challenges outlined above, we propose a novel time delay estimation method for traffic congestion propagation between roads using lag-specific transfer entropy (TE). Our main contributions are as follows: 
	\begin{itemize}
		\item We provide a model-free approach to estimate congestion propagation delays using a lag-specific TE estimator in complex urban road systems. 
		\item We quantify uncertainty in time delay estimation using bootstrap techniques. This uncertainty quantification is employed to evaluate the reliability of time delay estimates and serves as a basis for hyperparameter optimization. 
		\item We show that decomposition and nonlinear normalization with a sliding window are effective time series preprocessing methods for revealing causal relationships between traffic speed data.
		\item We validate the proposed method through numerical simulations and real user trajectory data obtained from a major GPS navigation system in South Korea.
	\end{itemize}
	
The remainder of this paper is organized as follows: Section \ref{s2} provides an overview of related studies. Section \ref{s3} presents crucial background information pertaining to the proposed method. Section \ref{s4} outlines our proposed time delay estimation method. Sections \ref{s5} and \ref{s6} validate the proposed method using simulated and real congestion propagation data, respectively. Finally, concluding remarks are presented in Section \ref{s7}.

\section{Related Work} \label{s2}
Various topics have been studied regarding the estimation of time delays due to traffic accidents, including travel time delay \citep{habtemichael2015incident}, incident duration \citep{garib1997estimating,nam2000exploratory}, real-time crash identification \citep{zhu2022real}, and incident impact quantification \citep{cao2021quantification,lee2022quantifying}. However, unlike the aforementioned studies, we focused on the propagation delay of traffic congestion caused by accidents in a road network.

\subsection{Time Delay Estimation for Congestion Propagation}

Although our proposed method is the first to estimate time delays for congestion propagation in road traffic networks, time delay estimation (TDE) is not a new problem. In digital signal processing, TDE refers to the task of ascertaining the differences in arrival times between signals received at sensor array. 
The most widespread approach for TDE is cross-correlation \cite{knapp1976generalized, souden2009broadband}. 
Supposing that signals are received from two sensors, the delay between the sensors can be estimated using the time lag that maximizes the cross-correlation between filtered versions of the received signal. 

Limited attempts have been made to perform cross-correlation analyses with traffic speed data. Conventional TDE methods based on cross-correlation have been applied to a real road vehicle pass-by measurement to enable traffic monitoring using passive acoustic sensors \cite{marmaroli2012comparative}. 
Similarly, \cite{chandra2008cross} conducted a cross-correlation analysis to prove the existence of a significant relationship between the current value of speed at a specific station, as well as past speed values at upstream and downstream stations in a freeway traffic network. We refer to this method as time-lagged cross-correlation (TLCC). This approach is not applicable in the presence of nonstationarity. 

To quantify the TLCC level between two nonstationary time series at different scales, \cite{shen2015analysis} proposed a time-lagged detrended cross-correlation analysis approach. This method, referred to as DCCA, divides an entire time series into overlapping boxes to handle nonstationarity \cite{vassoler2012dcca}. 

The method proposed in the present study was compared with TLCC and DCCA for evaluation purposes.

\subsection{Traffic Causality Analysis}
Various traffic causal analysis methods have been developed to identify causal relationships among congested roads and detect congestion propagation patterns.
The authors of \cite{sun2005traffic} forecasted future traffic flow by ranking input variables to identify a subset of the Bayesian network as the set of cause nodes using the Pearson correlation coefficient. A two-step mining architecture has been proposed to capture the origin points of road traffic congestion \citep{chawla2012inferring}. TE was used to reveal the delay propagation network among multiple airports with a time series of airport delays \cite{xiao2020study}. However, the aforementioned approaches primarily focus on revealing causal relationships and do not provide information regarding estimated time delays during congestion propagation.

\section{Preliminary}\label{s3}

\subsection{Bootstrap for Markov Chains}\label{bootstrap}
Suppose that $\{X_t\}_{t\geq 1}$ is a stationary Markov chain with a finite state space $S = \{s_1,\ldots,s_n\}$, where $n\in\mathbb{N}$. Let $\mathbf{P} = (p_{ij})\in\mathbb{R}^{n\times n}$ be the transition probability matrix of the chain and the stationary distribution by $\boldsymbol{\pi} = (\pi_1, \dots, \pi_n)$. Thus, for any $1 \le i, j \le n$, $p_{ij}=P(X_{t+1}=s_j|X_t=s_i)$ and $\pi_i=P(X_t=s_i)$. 
Given a time series $\{X_1, \ldots, X_{L}\}$ of size $L$ from a stationary Markov chain, $\pi_i$ and $p_{ij}$ can be estimated as
\begin{equation}\label{markov}
	\begin{aligned}
		\hat{\pi}_i &= \frac{1}{L} \sum_{t=1}^{L} \mathbb{1} (X_t=s_i), \quad
		\hat{p}_{ij} = \frac{1}{\hat{\pi}_i L} \sum_{t=1}^{L} \mathbb{1} (X_t=s_i, X_{t+1}=s_j).
	\end{aligned}
\end{equation}
The bootstrap observations $\left\{X_1^*, \ldots, X_{L}^*\right\}$ can now be generated using the estimated transition matrix and marginal distribution in Eq.~\eqref{markov} \citep{kreiss2012bootstrap}. 
\begin{enumerate}[\hspace{0.5cm}1)]
	\item Generate a random variable $X_1^*$ from the discrete distribution on $\{1,\ldots,n\}$ that assigns mass $\hat{\pi}_i$ to $s_i$, $1\leq i\leq n$.
	\item Generate a random variable $X_{t+1}^*$ from the discrete distribution on $\{1,\ldots,n\}$ that assigns mass $\hat{p}_{ij}$ to $j$, $1\leq j \leq n$, where $s_i$ is the value of $X_{t}^*$.
	\item Repeat Step 2) until a simulated time series $\{X_1^*,\ldots, X_{L}^*\}$ has been obtained. 
\end{enumerate}

\subsection{Lag-Specific Transfer Entropy}\label{tde}
TE is a measurement of directed information flow \citep{schreiber2000measuring} based on the concept of Shannon entropy \citep{shannon1948mathematical} in the field of information theory. 
For a discrete random variable $I$ with probability distribution $p(i)$, the Shannon entropy represents the average number of bits required to optimally encode independent draws, calculated as follows:
\begin{equation}\label{shannon}
	H(I) =-\sum_{i} p(i)\log_2 p(i).
\end{equation}

Eq.~\eqref{shannon} can be easily extended to the concept of conditional entropy using two discrete random variables $I$ and $J$:
\begin{equation}\label{conditional}
	H(I|J)=-\sum\sum p(i,j)\log_2 p(i|j)
\end{equation}
This equation can be used to measure information flow between two discrete random variables.

Consider two discrete random variables, $I$ and $J$, with marginal probability distributions $p(i)$ and $p(j)$, and joint probability $p(i,j)$. Suppose both variables represent stationary Markov processes of orders $k$ and $l$, respectively. 
For the order $k$ Markov process $I$, Eq.~\eqref{shannon} can be extended to 
$$
H^{(k)}(I)=-\sum_i p(i_{t},i_{t-1}^{(k)})\log p(i_{t}|i_{t-1}^{(k)}),
$$
where $i_{t-1}^{(k)}=(i_{t-1},\ldots,i_{t-k})$. Analogously, the information flow from $J$ to $I$ is measured by quantifying the deviation from the generalized Markov property $p(i_t|i_{t-1}^{(k)}) = p(i_t|i_{t-1}^{(k)}, j_{t-u}^{(l)} )$ for an arbitrary source--target lag  $u$, as follows: 
\begin{equation}\label{lag} 
	T_{J\rightarrow I}^{(k,l)}(t,u)=\sum p(i_{t},i_{t-1}^{(k)},j_{t-u}^{(l)})\log\frac{p(i_{t}|i_{t-1}^{(k)},j_{t-u}^{(l)})}{p(i_{t}|i_{t-1}^{(k)})}.
\end{equation}
Eq.~\eqref{lag} preserves the computational interpretation of TE as an information transfer, which is the only relevant option in keeping with Wiener's principle of causality \citep{wibral2013measuring}.
The transfer entropy is known to be biased in small samples \citep{marschinski2002analysing}. To correct any bias, \citep{marschinski2002analysing} proposed the effective transfer entropy ($ETE$):
\begin{equation}\label{ete2}
	ETE_{J\rightarrow I}^{(k,l)}(t,u)=T_{J\rightarrow I}^{(k,l)}(t,u)-T_{J_{\text{shuffled}}\rightarrow I}^{(k,l)}(t,u).
\end{equation}
where $T_{J_{\text{shuffled}}\rightarrow I}^{(k,l)}$ indicates the transfer entropy using a shuffled version of time series $J$. The shuffling process randomly draws values from the original time series and realigns them to generate a new time series. Thus, shuffling eliminates any time series dependencies of $J$, as well as statistical dependencies between $J$ and $I$. Note that $T_{J_{\text{shuffled}}\rightarrow I}^{(k,l)}$ converges to zero as the sample size increases, and any nonzero value of $T_{J_{\text{shuffled}}\rightarrow I}^{(k,l)}(t,u)$ is a result of the small sample effect. To ensure estimation consistency, shuffling is repeated, and the average of the shuffled transfer entropy estimates across all iterations serves as an estimator for the small sample bias, which is subsequently subtracted from the Shannon or Rényi transfer entropy estimate to correct any bias.

\section{Methodology}\label{s4}

\subsection{Problem Definition}
Suppose that congestion information is transferred from a source road to a destination road with a time delay of $u$. The objective of time delay estimation is to estimate $u$ given previously observed traffic speed data on the two roads, denoted as $\{X_t\}^L_{t=1}$ and $\{Y_t\}^L_{t=1}$, respectively. Therefore, the time delay estimation task entails formulating a function $f(\cdot)$ that computes the source--target lag $u$,
	$ \left[\{X_t\}_{t=1}^L, \{Y_t\}_{t=1}^L \right] \xrightarrow[]{f(\cdot)} u.$

The proposed time delay estimation algorithm comprises three steps. First, bootstrapping for each time series is performed using preprocessing methods. 
Next, the transfer entropy computation provides the estimated time delay lag $u^*$. Finally, the distribution of time delay lag is estimated to determine the existence of a statistical causal relationship.

\subsection{Preprocessing and Bootstrapping}\label{boot}
Consider a time series of congested traffic speed data, $\{X_t\}^L_{t=1}$, which has the properties of scale dependence, nonlinearity, and nonstationarity. 
To effectively identify the causal relationship within such a time series, appropriate preprocessing methods are essential. 

To this end, we first decompose the time series into a trend and its residual, as follows:
\begin{equation}\label{trend}
	\forall t, X_t = \mathcal{T}_{t} + R_{t} = \frac{1}{m} \sum_{j=0}^{m-1} X_{t-j} + R_{t},
\end{equation}
where $\mathcal{T}_{t}$ and $R_{t}$ are the trend and residual components, respectively. The trend component is a moving average of order $m$, representing the mean forefront value at time $t$.
The purpose of the trend component is to smooth the time series for estimating the underlying trend. After extracting the underlying trend from $\{X_t\}^L_{t=1}$, the residual time series $\{R_t\}^L_{t=1}$ is assumed to be a stationary Markov process. 
\textcolor{black}{The assumption of Markovian property in traffic speed is not a novel notion. Many prior traffic speed prediction and modeling studies have been conducted under this assumption \citep{hong2006highway, chandra2009predictions, vlahogianni2014short, pavlyuk2017short, song2019short}.} Based on $\{R_t\}^L_{t=1}$, we can generate the bootstrap residuals $\{R^{*(b)}_t\}^L_{t=1}$ as explained in Section \ref{bootstrap}. Subsequently, we can easily obtain a bootstrap time series $\{X^{*(b)}_t\}^L_{t=1}$ by $\mathcal{T}_{t}+R^{*(b)}_t$ for $t=1,\ldots,L$.

Nonlinear normalization with a sliding window is then applied to the obtained time series to address the scale-dependency, nonlinearity, and nonstationarity of traffic speed data. To ensure the data are scale-independent and close to linear \citep{wang2019desaturated}, the standard normal cumulative distribution function $\Phi$ is applied. A sliding window technique has similarly been employed to handle a nonstationary time series in \citep{ogasawara2010adaptive}. 
Let $\mathbf{X}_{t,w}=\left\{X^{*(b)}_k\right\}_{k=t-w+1}^t$ be the forefront sequence of $X^{*(b)}_t$ with a sliding window size of $w$, and $F_{25,t}$, $F_{50,t}$, and $F_{75,t}$ be the 25th, 50th, and 75th percentiles of $\mathbf{X}_{t,w}$, respectively. Note that these percentiles depend on the location of the sliding window. \textcolor{black}{Then, a normalized time series $\left\{\tilde{X}^{*(b)}_t\right\}^L_{t=1}$ can be obtained by} 
\begin{equation}\label{normalize}
	\tilde{X}^{*(b)}_t= \Phi\left(0.5\times\frac{X^{*(b)}_t-F_{50,t}}{F_{75,t}-F_{25,t}}\right).
\end{equation}

To verify the effectiveness of the nonlinear normalization method expressed in Eq.~\eqref{normalize}, we compared its performance with that of existing normalization methods \citep{wang2019desaturated, ogasawara2010adaptive}, including the min-max method $\left(\tilde{X}^{*(b)}_t=\frac{X^{*(b)}_t}{\max\mathbf{X}_{t,w}}\right)$ and the z-score method $\left(\tilde{X}^{*(b)}_t=\frac{X^{*(b)}_t-\mu(\mathbf{X}_{t,w})}{\sigma(\mathbf{X}_{t,w})}\right)$ with and without a sliding window technique (see Section \ref{s5}).

\subsection{Time Delay Estimation}\label{sym}

Using Eq.~\eqref{ete2}, the time lag in a causal relationship $J\rightarrow I$ can be estimated by solving the following optimization problem: 
\begin{equation}\label{opt}
	\hat{u}=\underset{u \in \mathbb{N}}{\operatorname{argmax}}\;ETE_{J\rightarrow I}^{(k,l)}(t,u).
\end{equation}
In this study, we assume $k=\ell=1$. 

To compute the lag-specific ETE in Eq.~\eqref{opt}, we discretize continuous data using symbolic encoding. This discretization can be performed by partitioning the data into a finite number of bins. We denote the bounds specified for the $n$ bins by $q_1,q_2,\ldots,q_{n-1}$, where $q_1<q_2<\cdots<q_{n-1}$. For the normalized time series in Eq.~\eqref{normalize}, we obtain the encoded time series $\{J^{*(b)}_t\}^L_{t=1}$ by the following equation:
\begin{equation}\label{encode}
	J^{*(b)}_t=\begin{cases} 
		1 &\mbox{for } \tilde{X}^{*(b)}_t\leq q_1 \\
		2 & \mbox{for } q_1<\tilde{X}^{*(b)}_t<q_2 \\
		\vdots & \quad\vdots \\
		n & \mbox{for } \tilde{X}^{*(b)}_t\geq q_{n-1}.
	\end{cases}
\end{equation}
The choice of bins depends on the distribution of data. In the case where tail observations are of particular interest, binning is typically based on empirical quantiles, such that the left and right tail observations are allocated into separate bins. 
In this study, we implemented symbolic encoding with $n=3$ based on 5\% and 95\% empirical quantiles, thereby emphasizing speed extremes caused by dynamic speed changes and traffic accidents.

Consequently, we obtain $\{J^{*(b)}_t\}^L_{t=1}$ from $\{\tilde{X}^{*(b)}_t\}^L_{t=1}$, and $\{I^{*(b)}_t\}^L_{t=1}$ from $\{\tilde{Y}^{*(b)}_t\}^L_{t=1}$, respectively, for $b=1,\ldots,B$. Given $\{J^{*(b)}_t\}^L_{t=1}$ and $\{I^{*(b)}_t\}^L_{t=1}$, $b=1,\ldots,B$, we obtain bootstrap observations of the time lag, $u^{*(1)}, \ldots,u^{*(B)}$ using Eq.~\eqref{opt}.

\subsection{Uncertainty Quantification of Time Delay Estimates}\label{sf}

Suppose that bootstrap observations of the time lag follow a distribution $\mathcal{G}$,
	$$
	u^{*(1)}, \ldots,u^{*(B)}\sim \mathcal{G},
	$$
which is unknown in practice. Let $\mu$ and $\sigma^2$ denote the mean and variance of $\mathcal{G}$, respectively, which can be estimated by 
	$$
	\hat{\mu}_B=\frac{1}{B}\sum_{b=1}^B u^{*(b)},\quad
	\hat{\sigma}_B^2=\frac{1}{B}\sum_{b=1}^B \left(u^{*(b)}\right)^2-\hat{\mu}_B^2.
	$$
Proposition \ref{p1} implies that 1) the bootstrap estimate $\hat{\mu}_B$ is an unbiased estimate of $\mu$, and 2) $\frac{1}{B}\hat{\sigma}_B^2$ quantifies the uncertainty of $\hat{\mu}_B$. That is, we can evaluate the uncertainty of the bootstrap estimate $\hat{\mu}_B$ using $\frac{1}{B}\hat{\sigma}_B^2$, which is practically useful because $\mu$ is unknown. This approach can be applied to hyperparameter tuning. In this study, we used a grid search to determine a set of hyperparameters (length of time series ($L$) and sliding window size ($w$)) that minimizes $\frac{1}{B}\hat{\sigma}_B^2$.

\begin{lemma}\label{p1}
	Let $u^{*(1)}, \ldots,u^{*(B)}$ be a bootstrap sample, and $E(u^{*(b)})=\mu$, $Var(u^{*(b)})=\sigma^2$. Then, the sample mean $\hat{\mu}_B=\frac{1}{B}\sum_{b=1}^B u^{*(b)}$ approximately follows $\mathcal{N}\left(\mu,\frac{1}{B}\hat{\sigma}_B^2\right)$, where $\hat{\sigma}_B^2$ is the sample variance of the bootstrap sample.    
\end{lemma}
\begin{proof}
	As $\hat{\sigma}_B^2\rightarrow \sigma^2$ in probability, 
	$\frac{\sqrt{B}(\hat{\mu}_B-\mu)}{ \hat{\sigma}_B}=\frac{\sigma}{ \hat{\sigma}_B}\frac{\sqrt{B}(\hat{\mu}_B-\mu)}{\sigma}\xrightarrow[]{d}\mathcal{N}(0,1)$ by the Central Limit Theorem and Slutsky's Theorem \citep{casella2021statistical}. Thus, $\hat{\mu}_B$ approximately follows a normal distribution,
	$\hat{\mu}_B\sim \mathcal{N}\left(\mu,\frac{1}{B}\hat{\sigma}_B^2\right)$. \end{proof}

To determine whether the bootstrap estimate $\hat{\mu}_B$ is reliable, we compare $\hat{\sigma}_B^2$ with a predetermined threshold $\sigma_T^2$. That is, if $\hat{\sigma}_B^2>\sigma_T^2$, we can conclude that $\hat{\mu}_B$ is not reliable, and congestion is therefore not propagated on the corresponding road segment.

To determine an appropriate value of $\sigma_T^2$, we employ the concept of the tolerance interval ($TI$). The $TI$ is a statistical interval in which a specified proportion $\gamma$ of a population will fall with a certain level of confidence $(1-\alpha)$.
	By definition, a $(\gamma,1-\alpha)$-$TI$ of $\hat{\mu}_B$, 
	$$
	TI=\left[\mu - k_{\gamma,\alpha}\sqrt{\frac{1}{B}\hat{\sigma}_B^2},\mu + k_{\gamma,\alpha}\sqrt{\frac{1}{B}\hat{\sigma}_B^2}\right],
	$$
	satisfies 
	$$
	P\left[P(\hat{\mu}_B\in TI) \ge \gamma\right] = 1-\alpha,
	$$
where $k_{\gamma,\alpha}$ is the tolerance factor \cite{witkovsky2014exact}. Then, $\sigma_T^2$ can be determined by $k_{\gamma,\alpha}\sqrt{\frac{1}{B}\sigma_T^2}=1$ (min) to yield a $\pm 1$ minute $TI$. With $B=100$, $\gamma=0.9$, and $\alpha=0.01$, the present study uses $\sigma_T^2=\frac{100}{k^2_{0.9,0.01}}=5.05^2$.

Consider the propagation of traffic congestion caused by accidents. Here, we define the propagation path by a sequence of incoming roads in the direction opposite to the traffic flow, where $k$th element in the propagation path is denoted as Hop($k$-1). Let $\hat{\mu}_{B, k-1}$ and $\hat{\sigma}^2_{B, k-1}$ denote the bootstrap estimate and sample variance at Hop($k$-1), respectively. Hop$0$ corresponds to the road where the accident occurred. We consider Hop($k$-1) to be statistically \textit{significant} if $\hat{\mu}_{B, k-1}$ and $\hat{\sigma}^2_{B, k-1}$ satisfy the following conditions: 1) $\hat{\sigma}^2_{B, k-1}<\sigma^2_T$ and 2) $\hat{\mu}_{B, k-2}<\hat{\mu}_{B, k-1}$. The second condition states that the time delay between Hop$0$ and Hop($k$-1) must exceed that between Hop$0$ and Hop($k$-2).

\section{Simulation Studies}\label{s5}
\begin{figure*}[!h]
	\centering
	\subfloat[Simulated data]{
		\includegraphics[clip,height=5.5cm]{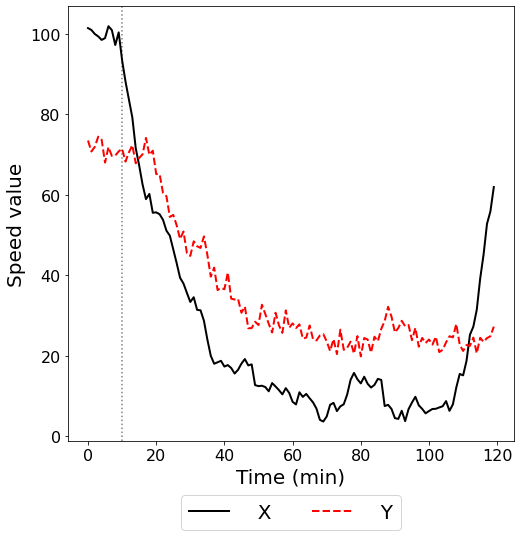}
	}
	\subfloat[Proposed method]{
		\includegraphics[clip,height=5.5cm]{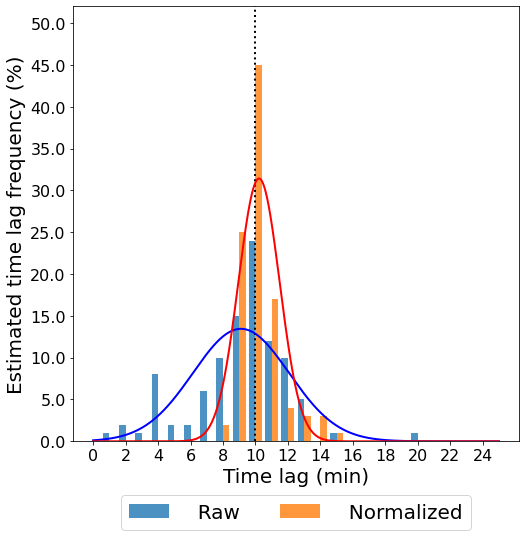}}
	\subfloat[DCCA method]{
		\includegraphics[clip,height=5.5cm]{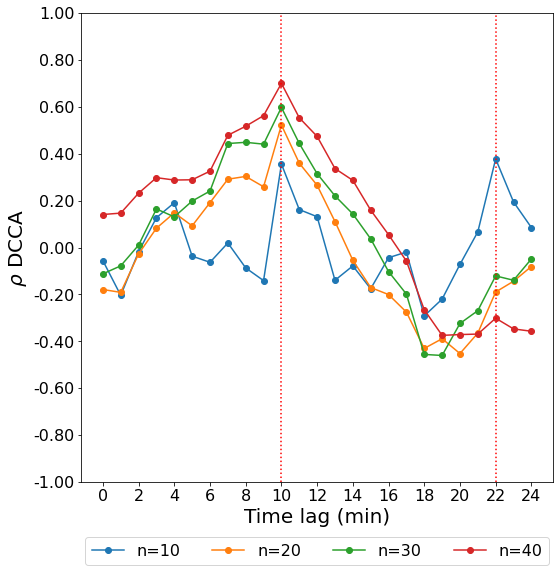}}
	\caption{Results of simulation study}
	\label{fig:norm_p}
\end{figure*}

The proposed method was validated using simulated data. Two time series -- $\{X_t\}_{t=1}^{120}$ and $\{Y_t\}_{t=1}^{120}$ -- were generated by
\begin{align*}\label{e1}
	X_t&=\begin{cases} 
		100 + \epsilon_{x,t} & \mbox{for } t < 10 \\
		0.95X_{t-1}+\epsilon_{x,t} &\mbox{for } 10\le t<95 \\
		1.10X_{t-1}+\epsilon_{x,t} & \mbox{for } t\le 120
	\end{cases}, \\
	Y_t&=\begin{cases}
		70 + \epsilon_{y,t} & \mbox{for } t<10\\
		0.5X_{t-u_0}+20+\epsilon_{y,t} &\mbox{for }t \geq 10
	\end{cases},
\end{align*}
where $\epsilon_{x,t}\sim\mathcal{N}(0,2)$ and $\epsilon_{y,t}\sim\mathcal{N}(0,2)$. A predetermined source--target lag ($u_0$) exists such that a significant information flow from $X$ to $Y$ is formed, but not vice versa. Figure \ref{fig:norm_p}(a) depicts two time series with $u_0 = 10$, where the black solid and red dashed lines represent $\{X_t\}_{t=1}^{120}$ and $\{Y_t\}_{t=1}^{120}$, respectively. This simulation represents a typical congestion propagation scenario between two adjacent roads $R_X$ and $R_Y$, assuming that there was a traffic accident on $R_X$ at $t = 10$ and that congestion was resolved at $t = 95$. With the time shift $u_0 = 10$, the congestion on $R_X$ propagates to $R_Y$. 

The proposed method was applied to simulated data with $m=2$ and $w=20$, as described in Section \ref{s4}. 
Figure \ref{fig:norm_p}(b) compares the distributions of bootstrap observations obtained from two-time series without normalization, to those obtained from two time series with nonlinear normalization. 
The red and blue lines in the figure denote the values of $\left(\hat{\mu}_B,\hat{\sigma}_B^2\right)$ in a distribution form, that is, $\left(9.54, 3.94^2\right)$ and $\left(10.30, 1.35^2\right)$, respectively. Here, a normal distribution is used for visualization purposes.
We confirmed that nonlinear normalization with a sliding window improved the accuracy of time delay estimation, as $1.35^2<3.94^2$.

\begin{table}[h]
\centering
\caption{Simulation results with comparison methods}\label{tab: sim}
\begin{tabular}{@{}cccccc@{}}
\toprule
 & \textbf{TLCC} & \textbf{DCCA(10)} & \textbf{DCCA(20)} & \textbf{DCCA(30)} & \textbf{DCCA(40)} \\ \midrule
$\hat{u}$ & 11.00   & 22.00       & 10.00       & 10.00       & 10.00       \\ \bottomrule
\end{tabular}
\end{table}

For comparison purposes, the TLCC and DCCA methods were also applied to the simulated data. The DCCA method requires a hyperparameter $n$, which indicates the size of the overlapping box \cite{vassoler2012dcca}. Here, we used multiple values of $n$ ($10,20,30,40$) to obtain the results of time delay estimation, as shown in Figure \ref{fig:norm_p}(c) for the DCCA method. Furthermore, Table \ref{tab: sim} summarizes the results of conventional TDE methods. These results show that both the TLCC and DCCA methods generally yield reasonable time delay estimates for the simulated data. In particular, it is recommended to set $n$ to be greater than 20.

\begin{table*}[htb!]
	\caption{Simulation results comparison ($u_0=10$) with $B=100$}\label{tab: simulation}
	\centering
	\begin{tabular}{@{}cccrrrrrr@{}}
		\toprule
		\multirow{2}{*}{\textbf{Decomposition}} & \multirow{2}{*}{\textbf{Normalization}}     & \multirow{2}{*}{\textbf{Metrics}} & \multicolumn{5}{c}{\textbf{Window length}}       & \multicolumn{1}{c}{\multirow{2}{*}{\textbf{Average}}} \\
		&                                             &                                   & \multicolumn{1}{c}{\textbf{10}} & \multicolumn{1}{c}{\textbf{20}} & \multicolumn{1}{c}{\textbf{30}} & \multicolumn{1}{c}{\textbf{40}} & \multicolumn{1}{c}{\textbf{120 (all)}} & \multicolumn{1}{c}{}                                  \\ \midrule
		\multirow{14}{*}{\textbf{false}}        & \multirow{3}{*}{\textbf{none}} & $\hat{\mu}_B$                     & -                               & -                               & -                               & -                               & 11.23                                  & 11.23                                                 \\
		&                                             & $\hat{\sigma}^2_B$                & -                               & -                               & -                               & -                               & 7.03                                   & 7.03                                                  \\
		&                                             & MAE                               & -                               & -                               & -                               & -                               & 6.13                                   & 6.13                                                  \\ \cmidrule(l){2-9} 
		& \multirow{3}{*}{\textbf{min-max}}           & $\hat{\mu}_B$                     & 12.93                           & 12.64                           & 13.14                           & 12.71                           & 14.89                                  & 13.26                                                 \\
		&                                             & $\hat{\sigma}^2_B$                & 7.21                            & 7.49                            & 7.27                            & 7.39                            & 6.72                                   & 7.21                                                  \\
		&                                             & MAE                               & 6.73                            & 6.94                            & 6.92                            & 6.82                            & 7.20                                   & 6.92                                                  \\ \cmidrule(l){2-9} 
		& \multirow{3}{*}{\textbf{z-score}}           & $\hat{\mu}_B$                     & 13.06                           & 13.51                           & 12.97                           & 13.29                           & 14.84                                  & 13.53                                                 \\
		&                                             & $\hat{\sigma}^2_B$                & 6.98                            & 6.97                            & 7.00                            & 7.15                            & 6.73                                   & 6.97                                                  \\
		&                                             & MAE                               & 6.49                            & 6.67                            & 6.52                            & 6.75                            & 7.23                                   & 6.73                                                  \\ \cmidrule(l){2-9} 
		& \multirow{3}{*}{\textbf{nonlinear}}         & $\hat{\mu}_B$                     & 13.29                           & 12.86                           & 12.28                           & 13.17                           & 14.27                                  & 13.17                                                 \\
		&                                             & $\hat{\sigma}^2_B$                & 7.12                            & 7.24                            & 7.13                            & 7.07                            & 6.89                                   & 7.09                                                  \\
		&                                             & MAE                               & 6.77                            & 6.66                            & 6.36                            & 6.63                            & 7.03                                   & 6.69                                                  \\ \midrule
		\multirow{14}{*}{\textbf{true}}         & \multirow{3}{*}{\textbf{none}} & $\hat{\mu}_B$                     & -                               & -                               & -                               & -                               & 9.54                                   & 9.54                                                  \\
		&                                             & $\hat{\sigma}^2_B$                & -                               & -                               & -                               & -                               & 3.94                                   & 3.94                                                  \\
		&                                             & MAE                               & -                               & -                               & -                               & -                               & 2.45                                   & 2.45                                                  \\ \cmidrule(l){2-9} 
		& \multirow{3}{*}{\textbf{min-max}}           & $\hat{\mu}_B$                     & 14.12                           & 10.57                           & 8.88                            & 10.28                           & 16.97                                  & 12.16                                                 \\
		&                                             & $\hat{\sigma}^2_B$                & 4.24                            & 2.81                            & 2.51                            & 3.53                            & 4.94                                   & 3.61                                                  \\
		&                                             & MAE                               & 4.75                            & 1.88                            & 1.95                            & 2.26                            & 7.38                                   & 3.64                                                  \\ \cmidrule(l){2-9} 
		& \multirow{3}{*}{\textbf{z-score}}           & $\hat{\mu}_B$                     & {10.09}                  & 10.28                           & 10.63                           & 11.97                           & 16.83                                  & 11.96                                                 \\
		&                                             & $\hat{\sigma}^2_B$                & 2.87                            & 3.58                            & 3.90                            & 4.19                            & 5.90                                   & 4.09                                                  \\
		&                                             & MAE                               & 1.58                            & 2.26                            & 2.55                            & 2.75                            & 7.96                                   & 3.42                                                  \\ \cmidrule(l){2-9} 
		& \multirow{3}{*}{\textbf{nonlinear}}         & $\hat{\mu}_B$                     & 10.78                           & 10.30                           & 10.78                           & 10.99                           & 13.38                                  & 11.25                                                 \\
		&                                             & $\hat{\sigma}^2_B$                & 2.98                            & {1.35}                   & 1.49                            & 2.04                            & 6.72                                   & 2.91                                                  \\
		&                                             & MAE                               & 1.53                            & {0.94}                   & 1.25                            & 1.72                            & 6.04                                   & 2.30                                                  \\ \bottomrule
	\end{tabular}
\end{table*}

To investigate the proposed method’s performance, we conducted simulation experiments under various settings of (1) decomposition, (2) normalization, and (3) length of the sliding window. Performance was evaluated using $\hat{\mu}_B$, $\hat{\sigma}^2_B$, and MAE, where $\textit{MAE} = \frac{1}{B}\sum_{i=1}^B |\hat{u}_i - u_0|$. Values closer to $u_0$ indicate a more accurate estimate $\hat{\mu}_B$. Likewise, smaller values of $\hat{\sigma}^2_B$ and MAE indicate better results. As summarized in Table \ref{tab: simulation}, the decomposition technique improved the overall performance, and nonlinear normalization with $w=20$ generally performed better than all other normalization methods in terms of $\hat{\sigma}_B$ and MAE. This implies that nonlinear normalization with $w=20$ produced the most precise and accurate estimates among the tested schemes.

\section{Real Data Example: Accident-Driven Traffic Congestion Propagation}\label{s6}
\begin{figure}
	\centering
	\includegraphics[clip,width=1\linewidth]{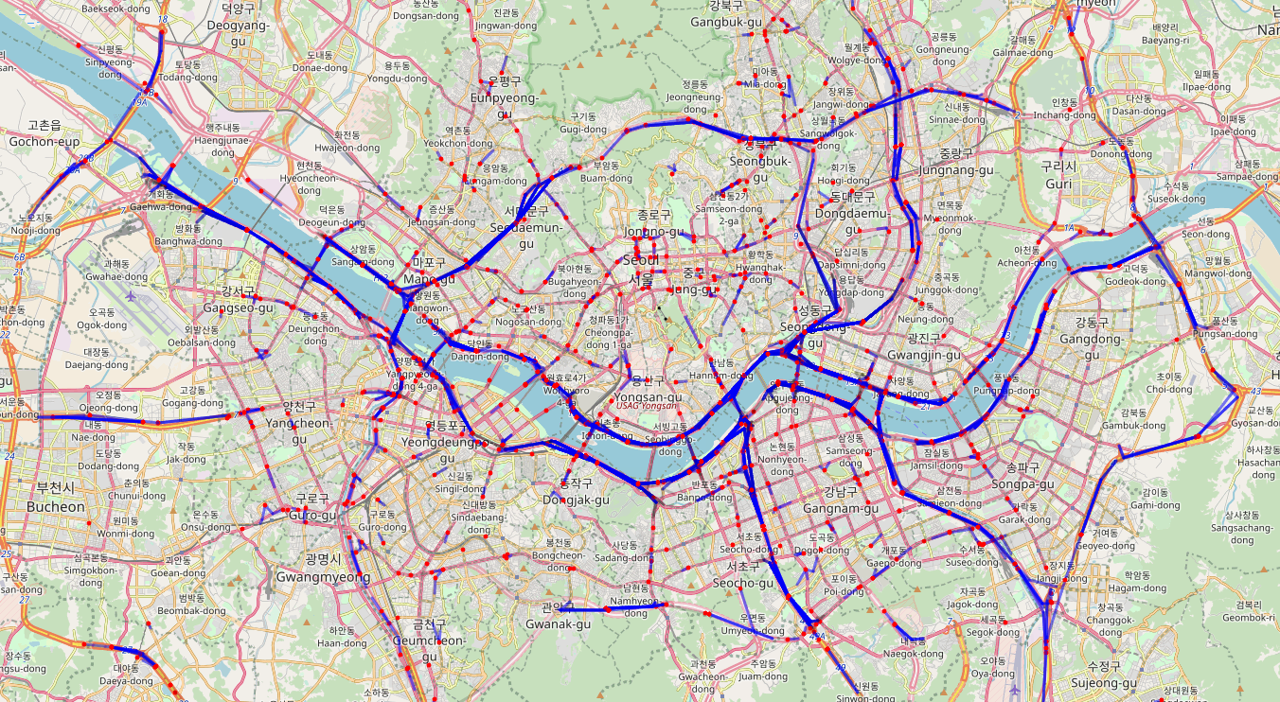}
	\caption{Locations of traffic accidents and their significant propagation paths in the city of Seoul from September 2020 to February 2021 (red dot: location of the accident, blue line: significant propagation path)}
	\label{fig:accident_map}
\end{figure}
Two types of datasets were considered from various sources: a traffic dataset provided by the NAVER corporation navigation team, and an accident dataset provided by the Korean National Police Agency. The traffic dataset encompasses trajectory-based speed and traffic road networks of the major metropolitan area of Seoul, where nearly half of the country's population resides. The speed data are described by GPS trajectories. A GPS trajectory consists of a series of points with latitude, longitude, and timestamp information generated during travel. To align a sequence of observed user positions within the road network, we used a map-matching process \citep{newson2009hidden}. Each accident record includes the reported time, information source, category, incident description, and point of origin described by both geographical coordinates and road segment ID, as discribed in Table \ref{t3}.

\begin{table}[h]
\centering
\caption{An example of a real accident record}\label{tab:incident_info}
\begin{tabular}{@{}ll@{}}
\toprule
\textbf{Event ID}           & \textbf{3786580}                                                                                                                     \\ \midrule
\textbf{Created datetime}   & 2021-01-30 19:00                                                                                                                     \\
\textbf{Information source} & Korean National Police Agency                                                                                                        \\
\textbf{Category}           & Accident                                                                                                                             \\
\textbf{Description}        & \begin{tabular}[c]{@{}l@{}}Traffic accident on the first   lane from Guro IC \\      on Nambu Belt Way to Anyang Bridge\end{tabular} \\
\textbf{Longitude}          & 126.87657                                                                                                                            \\
\textbf{Latitude}           & 37.48874                                                                                                                             \\ \bottomrule
\end{tabular}\label{t3}
\end{table}

The proposed method was validated on 3,197 real traffic accidents that occurred between September 2020 and February 2021 in Seoul. Figure \ref{fig:accident_map} presents the accident locations, as denoted by red stars, where the blue lines indicate significant propagation paths. 
Let $\hat{\mu}_{B, k-1}$ and $\hat{\sigma}^2_{B, k-1}$ denote the bootstrap estimate and sample variance at Hop($k$-1), respectively, for $k=2,3,4$. In this study, we investigated up to $k=4$. $k=1$ was excluded because Hop$0$ refers to the road where the accident occurred. 

\begin{table*}[htb]
\centering
\caption{Summary of time delay estimation results for 3,197 traffic accidents}\label{tab:stat}
\begin{tabular}{@{}ccccc@{}}
    \toprule
    \textbf{}      & \textbf{\begin{tabular}[c]{@{}c@{}}Number of \\      roads\end{tabular}} & \textbf{\begin{tabular}[c]{@{}c@{}}Significant \\      roads\end{tabular}} & \textbf{\begin{tabular}[c]{@{}c@{}}Significance \\      ratio\end{tabular}} & \textbf{\begin{tabular}[c]{@{}c@{}}Average \\      time delay (min)\end{tabular}} \\ \midrule
    \textbf{Hop$1$} & 5,036                                                                & 3,483                                                                     & 69.16\%                                                                    & 8.95                                                                            \\
    \textbf{Hop$2$} & 6,856                                                                & 4,479                                                                     & 65.33\%                                                                    & 11.10                                                                           \\
    \textbf{Hop$3$} & 9,721                                                                & 6,139                                                                     & 63.15\%                                                                    & 11.97                                                                           \\ \bottomrule
\end{tabular}
\end{table*}

Table~\ref{tab:stat} summarizes the time delay estimation results for all 3,197 traffic accidents. To ensure consistency within results, the hyperparameter $(w,L)=(60,180)$ was used for all accidents based on the grid search. There are 5,036 roads at Hop$1$ associated with accidents, approximately 69.16\% of which were revealed to be significant, with an average time delay of 8.95 minutes. For Hop$2$ and Hop$3$, 65.33\% and 63.15\% of the roads were revealed to be significant with average time delays of 11.10 and 11.97 minutes, respectively.

We selected two representative cases among the traffic data to detail how the proposed method identifies causal relationships and estimates time lag. Case 1 represents a simple road network with few propagation paths, whereas Case 2 represents a complex road network with many propagation paths. 
For comparison purposes, the TLCC and DCCA methods were also applied with equivalent settings to those used in the simulation study.

\subsection{Case 1: Simple Traffic Network}
The accident occurred on September 8, 2020, at 06:44 AM. The blue star in Figure \ref{fig:case1} denotes the exact location of the accident. Case 1 has one propagation path $[A,B,C,D]$. The black, red, blue, and green line segments in the figures indicate Hop$0$, Hop$1$, Hop$2$, and Hop$3$, respectively. Time delays were estimated using average speed data recorded at one-minute intervals from the previous hour to the subsequent two hours based on the time when the accident was reported.
\begin{figure}[h!]
	\centering
	\includegraphics[clip,width=0.5\linewidth]{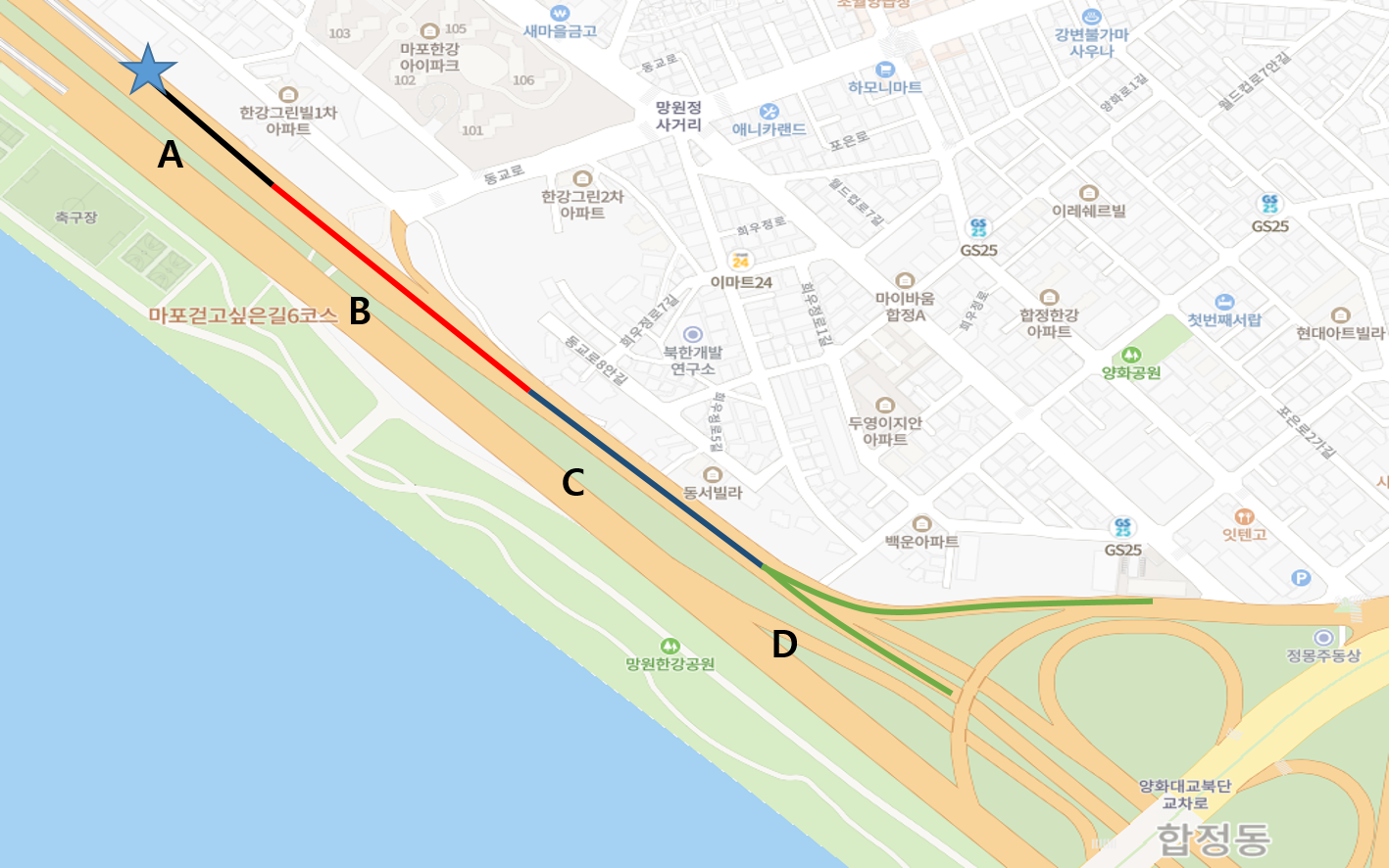}
	\caption{Traffic accident for Case 1}
	\label{fig:case1}
\end{figure}

\begin{figure*}[!ht]
	\centering
	\subfloat[Without normalization]{
		\includegraphics[clip,width=0.35\linewidth]{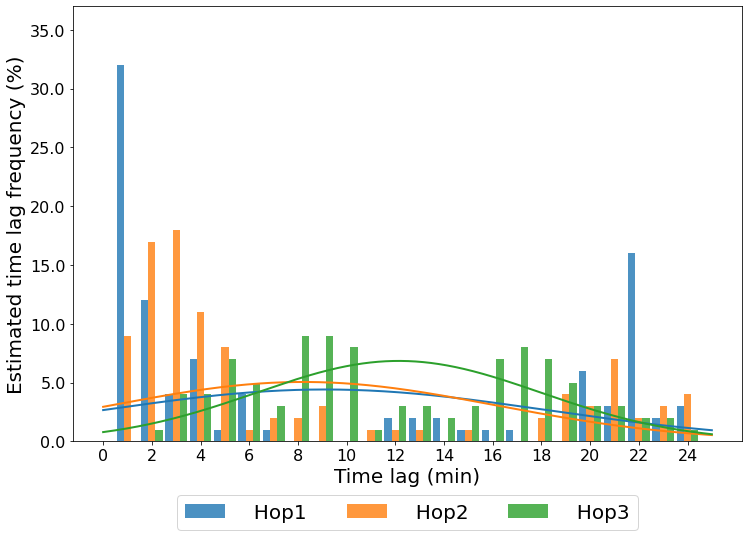}}\quad
	\subfloat[Nonlinear normalization]{
		\includegraphics[clip,width=0.35\linewidth]{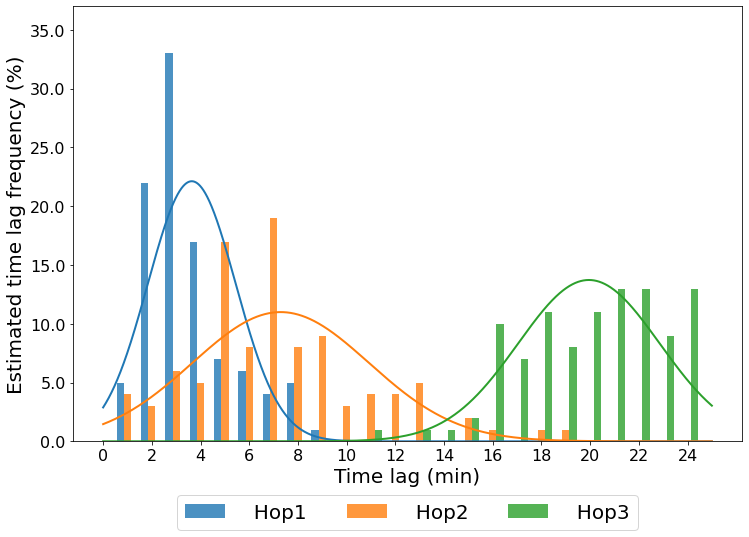}}
	\caption{Time delay estimation for Case 1}
	\label{fig:case_1}
\end{figure*}

Figures \ref{fig:case_1}(a) and \ref{fig:case_1}(b) show the results of time delay estimation for the propagation path $[A, B, C, D]$. 
It is apparent from the significantly smaller values of $\hat{\sigma}_B^2$ that the time series with nonlinear normalization produced a more consistent estimate of time delays that increased with each hop. We, therefore, conclude that the congestion effect of the accident propagated along the path to Hop1, Hop2, and Hop3 at 3.60, 7.30, and 19.97 min after the accident, respectively.

Unlike the proposed method, both TLLC and DCCA failed to identify the congestion propagation effect of the accident. Furthermore, the DCCA method produced inconsistent time delay estimates for varying values of $n$. Note that both the TLLC and DCCA methods do not provide uncertainty quantification of the time delay estimates.

\begin{table}[!htb]
\caption{Results of time delay estimation for Case 1}
\label{tab:case_1}
\centering
\begin{tabular}{@{}ccccccc@{}}
\toprule
\multirow{2}{*}{} &  \multicolumn{2}{c}{\textbf{Hop1}} & \multicolumn{2}{c}{\textbf{Hop2}} & \multicolumn{2}{c}{\textbf{Hop3}} \\ \cmidrule(lr){2-3} \cmidrule(lr){4-5} \cmidrule(l){6-7} 
 & $\hat{\mu}_B$     & $\hat{\sigma}_B^2$      & $\hat{\mu}_B$      & $\hat{\sigma}_B^2$      & $\hat{\mu}_B$      & $\hat{\sigma}_B^2$       \\ \midrule
\textbf{TLCC} & 7.00   & -    & 7.00 & -     & 0.00  & -     \\
\textbf{DCCA(10)} & 0.00   & -    & 19.00 & -     & 0.00  & -     \\
\textbf{DCCA(20)} & 0.00   & -    & 24.00 & -     & 17.00 & -     \\
\textbf{DCCA(30)} & 0.00   & -    & 10.00 & -     & 16.00 & -     \\
\textbf{DCCA(40)} & 0.00   & -    & 2.00  & -     & 8.00  & -     \\ \midrule
\textbf{without normalization}                                        & 9.62            & 83.70            & 8.07            & 61.31            & 12.08           & 34.83            \\
\textbf{nonlinear normalization}                                 & 3.60            & 2.88            & 7.30            & 13.83            & 19.97           & 8.93            \\ 
\bottomrule
\end{tabular}
\begin{flushleft}
\end{flushleft}
\end{table}

\subsection{Case 2: Complex Traffic Network}
The accident occurred on September 4, 2020 at 10:16 PM, and affected five propagation paths: $[A, B, C, D]$, $[A, E, F, G]$, $[A, H, I, J]$, $[A, H, K, M]$, and $[A, H, K, L]$. Figure \ref{fig:case2} depicts the exact location of the accident, along with the five propagation paths. For each path, the previous hour and the subsequent two hours were considered for time delay estimation. 
Figure \ref{fig:case_2} and Table \ref{tab: case2} present the results of time delay estimation. In Path 1, no specific causal relationship could be found, as shown in Figure \ref{fig:case_2}(a). This finding is supported by the corresponding high values of $\hat{\sigma}_B^2$ ($>\sigma^2_T=5.05^2$) in Table \ref{tab: case2}. Similarly, we can conclude that the congestion effect of the accident at road $A$ propagated along the second hop of Paths 2 and 3, and the third hop of Paths 4 and 5. Moreover, the values of $\hat{\sigma}_B^2$ in Table \ref{tab: case2} indicate that the congestion effect propagated along Path 3 up to Hop2, as depicted in Figure \ref{fig:case_2}(b), and along Paths 4 and 5 up to Hop3, as depicted in Figure \ref{fig:case_2}(c). For Paths 4 and 5, the time delay estimates are (8.23, 15.65, 22.06) and (8.22, 15.55, 20.75), respectively.
\begin{figure}[h!]
	\centering
	\includegraphics[clip,width=0.5\linewidth]{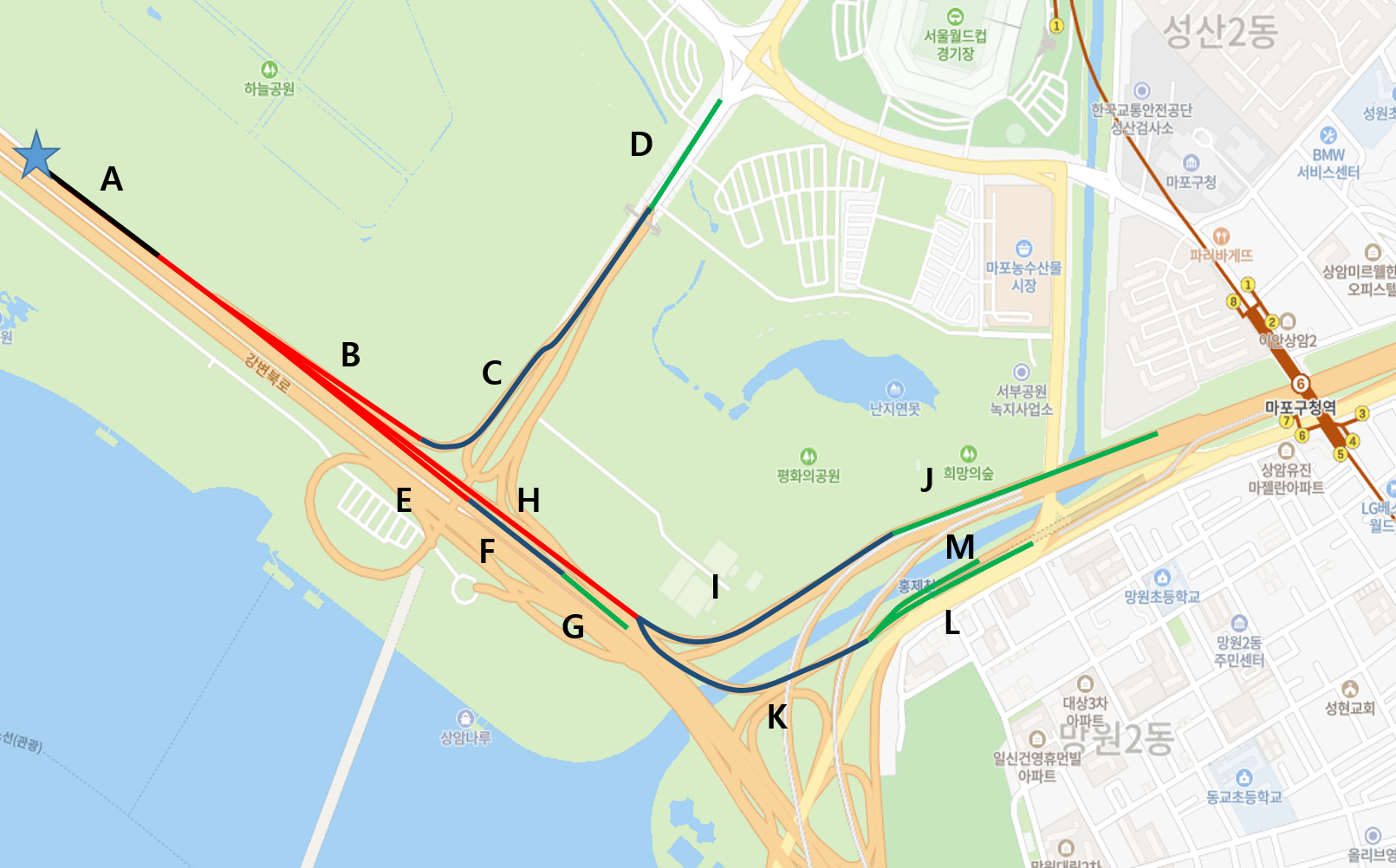}
	\caption{Traffic accident for Case 2}
	\label{fig:case2}
\end{figure}

\begin{figure*}[!htb]
	\centering
	\subfloat[Path 1]{
		\includegraphics[clip,width=0.333\linewidth]{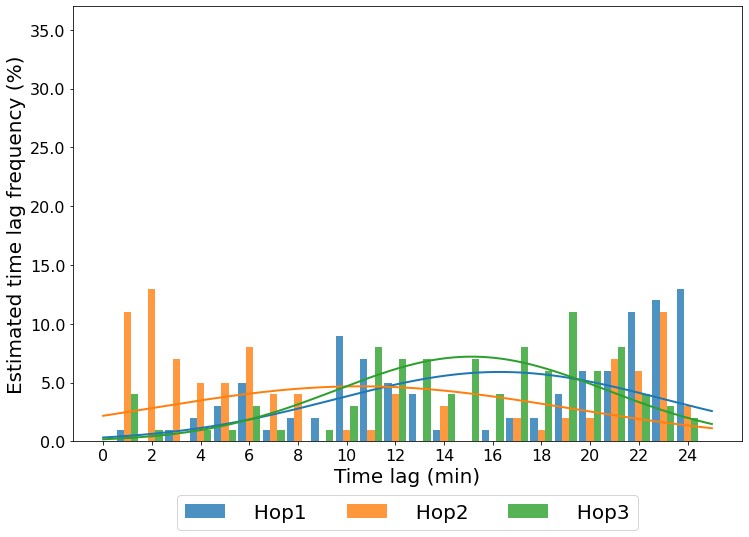}}
	\subfloat[Path 3]{
		\includegraphics[clip,width=0.333\linewidth]{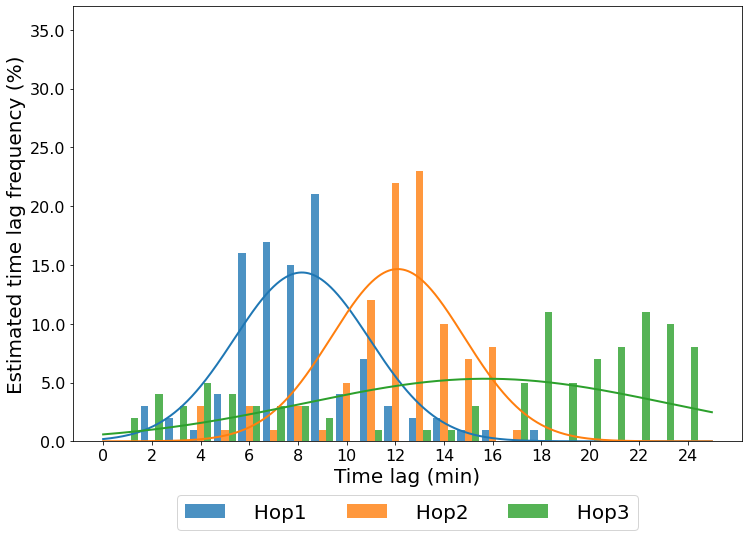}}
	\subfloat[Path 5]{
		\includegraphics[clip,width=0.333\linewidth]{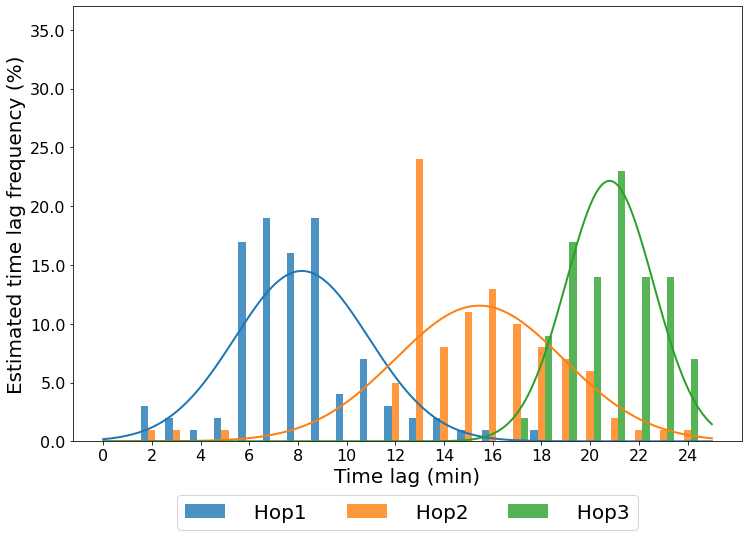}}
	\caption{Time delay estimation for Case 2}
	\label{fig:case_2}
\end{figure*}

\begin{table}[!htb]
\caption{Results of time delay estimation for Case 2}
\label{tab: case2}
\centering
\begin{tabular}{@{}clcccccc@{}}
\toprule
\multirow{2}{*}{\textbf{\begin{tabular}[c]{@{}c@{}}Propagation Path\end{tabular}}}      & \multicolumn{1}{c}{\multirow{2}{*}{\textbf{Methods}}} & \multicolumn{2}{c}{\textbf{Hop1}}  & \multicolumn{2}{c}{\textbf{Hop2}}  & \multicolumn{2}{c}{\textbf{Hop3}}  \\
                                                                                                & \multicolumn{1}{c}{}                                  & $\hat{\mu}_B$ & $\hat{\sigma}_B^2$ & $\hat{\mu}_B$ & $\hat{\sigma}_B^2$ & $\hat{\mu}_B$ & $\hat{\sigma}_B^2$ \\ \midrule
\multirow{3}{*}{\textbf{\begin{tabular}[c]{@{}c@{}}Path 1\\      $[A, B, C, D]$\end{tabular}}}  & \textbf{TLCC}                                         & 24.00         & -                  & 24.00         & -                  & 23.00         & -                  \\
                                                                                                & \textbf{DCCA(30)}                                         & 8.00          & -                  & 6.00          & -                  & 6.00          & -                  \\
                                                                                                & \textbf{Proposed}                                     & 16.03         & 47.09              & 11.59         & 73.56              & 15.17         & 30.46              \\ \midrule
\multirow{3}{*}{\textbf{\begin{tabular}[c]{@{}c@{}}Path 2 \\      $[A, E, F, G]$\end{tabular}}} & \textbf{TLCC}                                         & 22.00         & -                  & 12.00         & -                  & 4.00          & -                  \\
                                                                                                & \textbf{DCCA(30)}                                         & 2.00          & -                  & 5.00          & -                  & 5.00          & -                  \\
                                                                                                & \textbf{Proposed}                                     & 4.15          & 3.72               & 10.49         & 6.97               & 8.64          & 16.32              \\ \midrule
\multirow{3}{*}{\textbf{\begin{tabular}[c]{@{}c@{}}Path 3\\      $[A, H, I, J]$\end{tabular}}}  & \textbf{TLCC}                                         & 24.00         & -                  & 24.00         & -                  & 1.00          & -                  \\
                                                                                                & \textbf{DCCA(30)}                                         & 8.00          & -                  & 12.00         & -                  & 14.00         & -                  \\
                                                                                                & \textbf{Proposed}                                     & 8.21          & 7.97               & 12.13         & 7.03               & 15.72         & 57.16              \\ \midrule
\multirow{3}{*}{\textbf{\begin{tabular}[c]{@{}c@{}}Path 4\\      $[A, H, K, M]$\end{tabular}}}  & \textbf{TLCC}                                         & 24.00         & -                  & 24.00         & -                  & 24.00         & -                  \\
                                                                                                & \textbf{DCCA(30)}                                         & 8.00          & -                  & 2.00          & -                  & 19.00         & -                  \\
                                                                                                & \textbf{Proposed}                                     & 8.23          & 7.62               & 15.65         & 9.01               & 22.06         & 3.24               \\ \midrule
\multirow{3}{*}{\textbf{\begin{tabular}[c]{@{}c@{}}Path 5\\      $[A, H, K, L]$\end{tabular}}}  & \textbf{TLCC}                                         & 24.00         & -                  & 24.00         & -                  & 24.00         & -                  \\
                                                                                                & \textbf{DCCA(30)}                                         & 8.00          & -                  & 2.00          & -                  & 2.00          & -                  \\
                                                                                                & \textbf{Proposed}                                     & 8.22          & 7.67               & 15.55         & 11.01              & 20.75         & 3.17               \\ \bottomrule
\end{tabular}
\begin{flushleft}
\end{flushleft}
\end{table}

As in Case 1, both the TLCC and DCCA methods failed to estimate consistent time delays. The DCCA method with 30 overlapping boxes ($n=30$) obtained comparable results with the proposed method only for Path 3, as seen in Table~\ref{tab: case2}. From the results of both cases, it can be concluded that the proposed method identifies causal relationships and estimates the time lag more accurately than the conventional TDE methods.

\section{Conclusion}\label{s7}

Traffic congestion spreads its effects to the inflow roads, creating a causal relationship between the accident site and adjacent roads. To identify the said relationship, we developed a new method for estimating differences in congestion time. The proposed method utilizes a lag-specific TE estimator with decomposition and normalization techniques. Furthermore, we conducted extensive performance comparisons under varying experimental settings and found that the proposed decomposition and nonlinear normalization methods yield substantial performance improvements. We also confirmed that the proposed method produces more stable and robust results than the conventional TDE methods. Thus, the proposed time delay estimation method helps quantitatively understand the propagation of traffic congestion. 

Moreover, the bootstrap technique and its density estimation of statistical functionals enable the uncertainty quantification of time delay estimates. This uncertainty quantification allows us to evaluate the reliability of time delay estimates and serves as a basis for optimal hyperparameter tuning. Specifically, $\hat{\sigma}_B^2$ serves as a key indicator of a causal relationship between two-time series. We developed a rigorous and practical guidance for decision making based on the tolerance interval.

In this study, we only considered the method to obtain accurate time delay estimates using historical traffic data. Eventually, the proposed method will be used to predict time delays in GPS navigation systems, thereby providing users with more accurate estimated arrival times using real-time traffic data. Using our proposed method as a foundation, real-time delay prediction methods can be developed in future work.

\section*{Acknowledgments}

This work was partly supported by NAVER Corp. and the National Research Foundation of Korea (NRF) Grant funded by the Korea government (MSIT) (NRF-2021R1F1A1061038)

\section*{Conflict of interest}

The authors declare there is no conflict of interest.

\bibliographystyle{unsrt}
\bibliography{references}


\end{document}